\documentclass[letterpaper, 10 pt, conference]{ieeeconf}  
\IEEEoverridecommandlockouts                              

\usepackage{graphicx}
\usepackage{xcolor}
\usepackage{amsmath,amssymb}
\usepackage{bm}
\usepackage{bbm}
\usepackage{booktabs}
\usepackage{tikz}
\usetikzlibrary{calc,math}
\usepackage[export]{adjustbox}
\usepackage{siunitx}
\usepackage{cleveref}
\usepackage{algorithm}

\makeatletter
\expandafter\def\csname fst@algorithm\endcsname{%
  \fs@ruled
  \def\@fs@pre{%
    \vspace*{6pt}%
    \hrule height.8pt depth0pt \kern2pt}%
}
\makeatother

\usepackage[noEnd=true,indLines=true,commentColor=black]{sty/algpseudocodex}
\usepackage{sty/custom}
\usepackage[noadjust]{cite}

\title{\LARGE \bf
GAMBIT: Learning to Plan Continuous Multi-Robot Trajectories
}

\providecommand{\citep}[1]{\cite{#1}}

\author{Rishabh Jain$^{1}$, Akmaral Moldagalieva$^{2}$, Lorenzo Magnino$^1$, Michael Amir$^1$,\\%
Keisuke Okumura$^{1,3}$, Ajay Shankar$^{1}$, Wolfgang H{\"o}nig$^{2}$ and Amanda Prorok$^{1}$%
\thanks{This study was partially supported by Trinity College, Cambridge, European Research Council (ERC) Project 949940 (gAIa), JST PRESTO (JPMJPR2513) and the Deutsche Forschungsgemeinschaft (DFG, German Research Foundation) - 448549715.}%
\thanks{$^{1}$University of Cambridge, UK \texttt{\{rj412, lm2183, ko393, ma2151, as3233, asp45\}@cst.cam.ac.uk}}%
\thanks{$^2$Technical University of Berlin, Germany \texttt{\{moldagalieva, hoenig\}@tu-berlin.de}}%
\thanks{$^3$National Institute of Advanced Industrial Science and Technology (AIST), Japan}
}

\begin{document}
\maketitle
\thispagestyle{empty}
\pagestyle{empty}
\begin{abstract}
GAMBIT is an opening chess move in which a player sacrifices a piece, typically a pawn, to gain a positional advantage later in the game.
Analogously, in multi-robot coordination, individual robots may need to forgo locally reward-maximising behaviours to improve overall team performance.
Such self-sacrificial behaviours are difficult to capture with manually designed heuristics, particularly in dense, interaction-rich environments.
Focusing on double-integrator continuous dynamics, this work studies how to learn such coordinated heuristics over motion primitives for multi-robot trajectory execution.
Our framework, \gambit, first learns coordinated motion-primitive selection through imitation learning and subsequently fine-tunes the policy through reinforcement learning.
We further introduce a safeguarded rollout mechanism with backup trajectories that guarantees collision-free execution at all times.
Experiments demonstrate that \gambit substantially outperforms a range of baselines, including centralised motion planners and decentralised reactive planners, while exhibiting strong scalability.
In particular, it coordinates over a thousand robots with planning latency below a few hundred milliseconds in continuous domains.
\end{abstract}

\section{Introduction}
\label{sec:intro}
Generating trajectory plans using \emph{motion primitives} has been a successful paradigm for deriving potentially long-horizon, feasible motions to steer a robot---with applications ranging from autonomous driving to manoeuvres of aerial robots~\cite{likhachev2009planning,pivtoraiko2011kinodynamic,ortiz2024idb}.
While its multi-robot counterpart has also been studied~\cite{wen2022cl,moldagalieva2026dblacam}, a critical challenge in achieving fast and scalable planning remains the design of effective \emph{heuristics} that evaluate how ``good'' each primitive is.
Designing such a heuristic is challenging even in the single-robot case, as accurate cost-to-go estimates become difficult to obtain for higher-order dynamics beyond simple point-mass models; the multi-robot case further requires reasoning about coordination among agents.
One could design a trivial heuristic based on an interaction-agnostic single-robot heuristic, but such approximations can get significantly distorted in interaction-rich multi-agent environments, where effective coordination may require non-greedy motions, such as temporarily moving away from the goal to accommodate other agents.

Using na\"{i}ve or inaccurate heuristics for motion primitives can misguide the planning process, resulting in either considerable computational effort or highly suboptimal robot trajectories.
Manually capturing such coordination requires substantial engineering effort that scales poorly with the complexity of interactions.
This makes automatically \emph{learning} these heuristics particularly attractive.
This study investigates how to realise this idea.

Learning agent-wise heuristics to synthesise coordination has proven effective in the discrete-state, \emph{grid-world} \emph{multi-agent pathfinding (MAPF)} domain~\cite{jiang2025deploying,wang2025lns2+,jain2026graph}, where the objective is to find conflict-free paths for multiple agents on a four-connected grid.
Recent work has developed high-performing heuristics using decentralised neural MAPF policies~\cite{alkazzi2024a}, where each agent receives an observation and outputs an action distribution, trained using either reinforcement learning (RL) or imitation learning (IL) from expert data.
A natural question then arises: can we effectively transfer this knowledge to \emph{continuous} state spaces?

With some adaptation, this transfer appears possible using action synthesis based on motion primitives;
however, the transition to continuous state spaces introduces fundamental challenges.
In the context of IL, acquiring a substantial amount of expert data for scenarios involving many agents is considerably more difficult than in MAPF,
where established and reliable planners are readily available~\cite{okumura2024engineering},
a challenge we hereinafter refer to as \textbf{data scarcity}.
The \textbf{multi-modality} of solutions poses another challenge: trajectories belonging to different homotopy classes can have comparable costs, making it difficult to assign a unique target action to imitate.
The \textbf{safeguarding gap} in rollouts---i.e., the lack of mechanisms that clamp collision-causing actions and facilitate learning in strong neural MAPF policies~\cite{okumura2022priority,veerapaneni2024improving,jain2026pairwise}---weakens the reward signal, making training more difficult, particularly for RL.
Along with difficulties in \textbf{scenario generation} for providing informative training signals, these challenges collectively make it difficult to design advanced heuristics for multi-robot motion primitives.

{
\begin{figure*}[ht!]
\vspace*{2pt}
\centering
\includegraphics[width=1.0\linewidth]{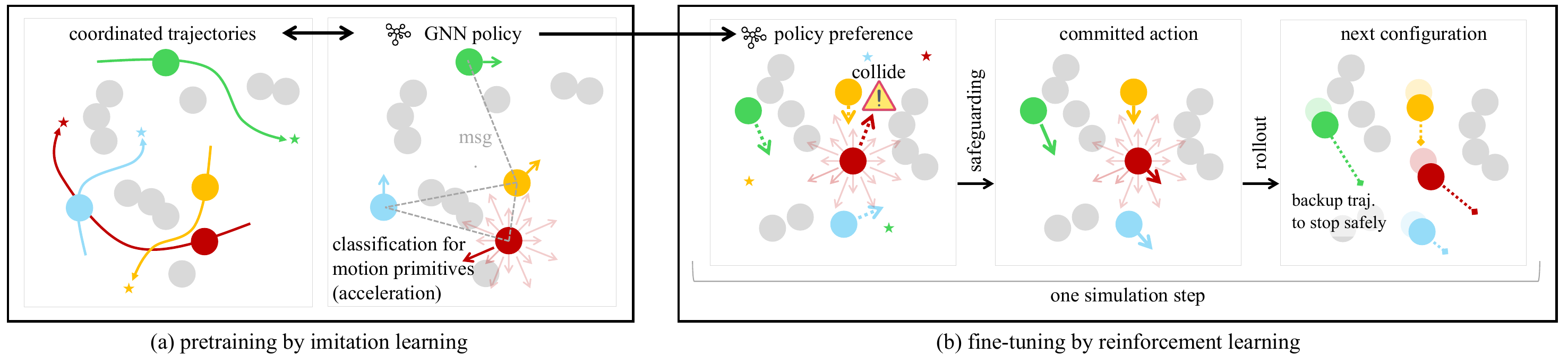}
\caption{
\gambit framework overview.
We aim to learn a coordination-aware GNN policy that provides heuristic preferences over motion primitives.
The framework consists of \emph{(a)}~imitation pretraining and \emph{(b)}~RL fine-tuning.
The former step learns from demonstration trajectories by treating motion-primitive selection as a classification task.
The latter uses multi-agent RL to maximise multi-robot navigation performance through our safeguarded rollout mechanism.
The resulting policy, combined with the safeguarded rollout, enables collision-free execution while retaining high scalability.
}
\label{fig:overview}
\end{figure*}
}

We address these challenges with a framework we term \gambit{}: \emph{\textbf{G}raph-\textbf{A}ttentive \textbf{M}otion primitives with \textbf{B}ackup for \textbf{I}nteractive \textbf{T}rajectories}.
Following the success of neural MAPF policies, our backbone architecture is based on a graph neural network (GNN) that outputs a categorical action distribution~\cite{li2021message,jain2026graph}, exploiting geometric locality to enable training on small-scale instances and deployment to larger teams.
The model is first pretrained via IL on instances with a small number of agents, using solutions generated by a centralised kinodynamic multi-robot motion planner~\cite{moldagalieva2026dblacam}, as sufficient data are attainable at this scale to overcome \textbf{data scarcity}.
The model is then fine-tuned via RL to mitigate the \textbf{multi-modality} problem by biasing the policy towards consistent action choices.
RL training uses newly developed safety-preserving rollouts with backup trajectories that overcome the \textbf{safeguarding gap} and stabilise the learning process.
Effective learning is further facilitated by carefully designed \textbf{scenario generation} that ensures agent interactions in every training instance.
\Cref{fig:overview} illustrates this pipeline.

The resulting \gambit policy, combined with motion primitives and our specialised rollout, synthesises safe continuous actions for hundreds of robots at sub-hundred-millisecond latency, while producing substantially better-coordinated trajectories in densely populated environments than reactive planners~\cite{van2011reciprocal,zhou2021ego}, computationally intensive centralised planners~\cite{zhang2025d4orm,moldagalieva2026dblacam}, and existing learning-based approaches~\cite{shaoul2025multi}.
We further validate \gambit through a physical deployment of eight multirotors~\cite{woo2025sanity} with closed-loop control.
Across both simulated and real-robot deployments, our backup mechanism guarantees safe and feasible execution at all times.

As a motion-primitive heuristic constitutes a lightweight and modular planning component, we view \gambit as a promising building block for designing capable planners beyond what we present here.
Furthermore, as policy rollout relies only on local interactions between robots, this scheme is also amenable to decentralised deployment.
In what follows, we present the problem formulation, training pipeline, safeguarding rollout, and experimental evaluation, in order.

\section{Preliminaries and Related Work}
\subsection{Problem Formulation}
\label{sec:problem-formulation}
We consider point-to-point navigation for a team of homogeneous spherical robots $A = \{1, \ldots, n\}$ with radius $R\sub{robot} \in \mathbb{R}^+$ under double-integrator dynamics, operating in a 2D or 3D workspace $W \subseteq \mathbb{R}^d$, $d \in \{2, 3\}$, containing obstacles.
Each robot $i\in A$ has its own start and goal positions $\bm{s}_i, \bm{g}_i \in W$.
Its state is given by its position and velocity, $\bm{x}_i = [\bm{p}_i, \bm{v}_i]$, while its control input is its acceleration, $\bm{u}_i = \bm{a}_i$.
The robot dynamics are given by
\begin{align}
\bm{x}_i^{k+1}
=
\begin{bmatrix}
\bm{p}_i^k
+ \Delta t\,\bm{v}_i^k
+ \frac{1}{2}\Delta t^2\,\bm{a}_i^k
\\
\bm{v}_i^k
+ \Delta t\,\bm{a}_i^k
\end{bmatrix},
\label{eq:dynamics}
\end{align}
where $\Delta t$ denotes the timestep.
This corresponds to the exact discretisation of the double-integrator dynamics under constant acceleration over each timestep.

Our objective is then to derive a kinodynamically consistent state-control trajectory of a priori unknown length $T \in \mathbb{N}^+$, $[\langle \bm{x}_i^0, \bm{u}_i^0 \rangle, \ldots, \langle \bm{x}_i^T, \bm{u}_i^T \rangle]$ for each robot to navigate from its start to goal without collisions with either obstacles or other robots, satisfying:
\begin{itemize}
\item $\bm{x}_i^0 = [\bm{s}_i, \bm{0}]$;
\item $\bm{x}_i^T = [\bm{p}_i^T, \bm{v}_i^T]$, with $\| \bm{p}_i^T - \bm{g}_i \| \leq R\sub{goal}$ and $\|\bm{v}_i^T\| \leq \epsilon$;
\item $\| \bm{p}_i^k - \bm{p}_j^k \| > 2R\sub{robot}$ for all $i \neq j \in A$; and
\item $\bm{p}_i^k$ does not collide with obstacles.
\end{itemize}
Here, $R\sub{goal} \in \mathbb{R}^+$ denotes the goal radius, $\bm{p}_i^k$ denotes the positional component of state $\bm{x}_i^k$, and $\epsilon$ denotes a sufficiently small number for the velocity threshold.
The solution quality is assessed by the \emph{sum-of-costs}, defined as the sum of travel times over all robots until each robot reaches its goal and remains there indefinitely.

Herein, we use `agent' and `robot' interchangeably.
A \emph{configuration} $Q = (\bm{x}_1, \ldots, \bm{x}_n)$ refers to a list of states for all agents.
We assume access to a collision-checking function $\funcname{collide}$ that determines whether a given state is in collision, following a standard assumption in motion planning~\cite{lavalle2006planning}.

\subsection{Planning with Motion Primitives}
A common approach to deriving such control sequences is via numerical optimisation, e.g., formulating trajectory optimisation as a non-linear program subject to dynamics and model constraints over a continuous control space, and solving it in a potentially decentralised manner~\cite{zhou2021ego,tordesillas2021mader}.
In contrast, this work follows a planning approach based on motion primitives, which restricts the infinitely many possible controls to a finite set, thus facilitating rapid discovery of coordinated actions.
Motion primitives have been explored for multi-robot navigation~\cite{moldagalieva2025db,hou2025primitive,moldagalieva2026dblacam}, demonstrating their potential for computationally efficient coordination, but they typically rely on manually designed mechanisms for selecting the right primitive.
\textbf{Our primary objective is to automate this selection process through learning.}
Accordingly, we distinguish \gambit from learning-based multi-agent control approaches that directly output continuous controls (e.g.,~\cite{riviere2020glas,batra2022decentralized}); instead, we formulate the learning problem as one of ranking the available primitives, making the learned heuristic naturally compatible with numerous global search-based planning methods.
Meanwhile, we also provide a safeguarded rollout scheme with motion primitives tailored to double integrator dynamics, demonstrating how the learned heuristic can be integrated into a practical deployment framework.

We define a \emph{motion primitive} as a finite sequence of control inputs $m = [\bm{u}^0, \bm{u}^1, \ldots, \bm{u}^{l-1}]$.
When applied to an initial robot state $\bm{x}^0$, a motion primitive induces a rollout state trajectory $[\bm{x}^0, \bm{x}^1, \ldots, \bm{x}^l]$.
In this study, for simplicity, we set $l = 1$; i.e., each primitive represents a single acceleration command $\bm{a} \in \mathbb{R}^d$, although the proposed framework can be extended to larger values of $l$.
We denote the set of available motion primitives by $M$.
Our objective is then to learn a \emph{ranking function} that ranks the primitives $m \in M$ for robot $i$ given the current configuration $Q$;
we interchangeably refer to this function as a \emph{heuristic}, as it virtually assesses the quality of each motion primitive.

While learning for motion-primitive-based planning has been studied in single-robot settings~\cite{camci2019learning,franke2025accelerating}, our primary interest lies in multi-robot systems, which present challenges that stem from the interplay between agents.

\subsection{Neural MAPF Policies}
Learning heuristics for multi-agent action selection has been studied extensively in the context of grid-world MAPF~\cite{alkazzi2024a}, under highly simplified discrete dynamics.
Typically, this line of work aims to design a policy that receives a local observation $o_i$ for agent $i$ and outputs a distribution over five discrete actions: \texttt{NORTH}, \texttt{WEST}, \texttt{SOUTH}, \texttt{EAST}, and \texttt{STAY}.
To obtain effective policies, established learning pipelines use either imitation learning (IL) or reinforcement learning (RL)~\cite{sartoretti2019primal,li2021message}.
We aim to leverage these insights beyond grid-world dynamics; however, several challenges hinder their direct adaptation to our setting, as highlighted in \cref{sec:intro}; that is, \textbf{data scarcity}, \textbf{multi-modality}, \textbf{safeguarding gap}, and \textbf{scenario generation}.
Below, we provide our \gambit approach to address these.

\section{\gambit pipeline}
The \gambit framework centres around a GNN-based policy, together with its training pipeline, comprising IL pretraining and RL fine-tuning, and a safety-preserving rollout scheme.
This section introduces the backbone policy architecture and then describes each component in turn.

\subsection{Policy Architecture}
While our pipeline is not restricted to a specific model architecture, we follow a successful paradigm adopted by GNN-based neural MAPF policies~\cite{li2021message,jain2026graph}.
The geometric locality of GNNs allows us to train the policy on instances with only a small number of robots while effectively capturing interactions between robots, and to deploy the learned policy to substantially larger teams at test time.
Specifically, the policy $\pi$ consists of
\emph{(i)}~an encoder $\phi\sub{enc}$,
\emph{(ii)}~graph neural network layers $\phi\sub{comm}$ that encode interactions between agents, and
\emph{(iii)}~an MLP decoder $\phi\sub{dec}$ that outputs an action distribution over the set of motion primitives $M$.

\subsubsection{Observation Encoding}
An observation for robot $i$, denoted by $o_i$, is the concatenation of
\emph{(a)}~its velocity,
\emph{(b)}~the relative goal vector,
\emph{(c)}~a field-of-view (FOV) representation of its surroundings $f_i$, and
\emph{(d)}~a motion-primitive embedding $\xi_i$.
The FOV $f_i$ is represented as a three-channel $d$-dimensional grid with $N\sub{obs}$ cells along each dimension.
The first two channels encode obstacle and agent occupancy, respectively, while the third encodes the goal location.
When the goal lies outside the FOV, its location is clipped to the FOV boundary, following prior learning-based MAPF approaches~\cite{sartoretti2019primal,li2021message}.
The motion-primitive embedding $\xi_i$ encodes the outcomes of the available primitives.
For each $m \in M$, we compute the relative displacement after executing $m$ followed by stopping, together with the corresponding relative goal vector.
Their concatenation is processed by an MLP and averaged over all primitives to obtain $\xi_i$.
Finally, the complete observation is encoded as $z_i^0 = \phi\sub{enc}(o_i)$.

\subsubsection{Interaction Modelling and Decoding}
For a configuration $Q$, we introduce a communication graph $G\sub{comm}$, where each vertex represents a robot $i \in A$ and an edge $i$--$j$ exists if robots $i$ and $j$ are sufficiently close in position, i.e., within a communication radius $R\sub{comm}$.
The communication module $\phi\sub{comm}$ propagates the robot-wise embeddings $z_i^0$ through $L$ GNN message-passing layers over $G\sub{comm}$ using scaled dot-product attention~\cite{Vaswani-Attention}, with pairwise relative positions as edge features.
Finally, based on the resulting representation $z_i^L$, the decoder $\phi\sub{dec}$ outputs an action distribution over $M$.

\subsection{Pretraining by Imitation Learning}
\label{sec:pretraining}

Pretraining by IL followed by RL fine-tuning has emerged as a successful strategy for learning capable policies across domains, ranging from language modelling~\cite{ouyang2022training} to robot learning~\cite{intelligence2026pi}; we follow this strategy.
For expert data collection, we employ an adapted version of db-LaCAM~\cite{moldagalieva2026dblacam}, an offline centralised planner, in which its low-level configuration generation is replaced with our enhanced PIBT procedure described in \cref{sec:safe-rollout}, enabling long-horizon kinodynamic multi-robot motion planning as defined in \cref{sec:problem-formulation}.
We generate demonstration trajectories by running the planner on instances involving a few agents.
We limit the team size because collecting demonstrations for large teams is computationally expensive even with leading planners in this domain, making the construction of a large corpus prohibitively costly.
Thanks to the locality of GNNs, we instead generate small instances that require tight local coordination, allowing the learned interactions to generalise to larger teams to overcome \textbf{data scarcity}.

\textbf{Scenario generation} is therefore a key component for successful learning.
In particular, we prepare two families of instances, \emph{(i)}~\mapname{Blocked Random} and \emph{(ii)}~\mapname{Forest}, visualised in \cref{fig:result}.
Given the number of robots, the former first places start--goal pairs at random positions and then creates a small wall between these positions by placing
obstacles perpendicularly to the line segments connecting each start--goal pair, with some random perturbations.
Unlike purely random instance generation, this structure is more likely to require coordination.
The latter contains regularly placed obstacles that form numerous narrow passages, while start--goal pairs are sampled randomly, exposing the policy to obstacle avoidance and potential deadlocks with other robots.

Our IL formulation treats motion primitive selection as a classification problem, rather than directly regressing continuous actions.
Accordingly, for each continuous expert action~$\bm{u}$, we construct its discrete target as the closest motion primitive $m^\ast \in M$ and train the policy to predict $m^\ast$ using cross-entropy loss.
This formulation offers two advantages over directly regressing continuous actions followed by quantisation: \emph{(i)}~more stable learning, as we observed empirically, and \emph{(ii)}~a categorical distribution over $M$, which can be directly exploited by our safeguarded rollout scheme.

\subsection{Fine-tuning by Reinforcement Learning}
While IL provides the basic skills required for multi-robot navigation, IL alone does not necessarily yield a skilful policy.
In addition to the solution \textbf{multi-modality}, this limitation is amplified by the fact that computing high-quality trajectories for motion planning is itself computationally challenging~\cite{solovey2020complexity}, and we therefore rely on suboptimal solutions obtained by expert planners within a practical time budget.
Consequently, the demonstration dataset can contain suboptimal behaviours, which are inherited by the policy through imitation.
We address this limitation through RL fine-tuning, allowing the policy to directly optimise solution quality.
Empirically, we find that training with RL from scratch is unstable, whereas the two-stage IL--RL procedure yields stable training and a capable policy.

For this RL phase, we use the same instance distribution as in the IL phase and employ MAPPO~\cite{yu2022surprising} with a simple reward directly focused
on optimising the sum-of-costs and an additional reward for all robots reaching their goals.
Let $Q$ denote the rollout configuration obtained by applying the actions selected by the policy to the input configuration.
For each robot $i$, we define the binary goal indicator,
$\mathbb{I}_i(Q) := \mathbb{I}\left[i \text{ is at its goal in } Q\right]$.
The reward is then defined as
\begin{align}
w_1 \prod_{i \in A} \mathbb{I}_i(Q) - w_2 \sum_{i \in A} \left(1-\mathbb{I}_i(Q)\right),
\label{eq:reward}
\end{align}
where $w_1,w_2>0$ are weighting coefficients.
The first term provides a team-level bonus when all robots reach their goals, encouraging successful planning as a team.
The second term penalises each robot that has not yet reached its goal at every timestep, thereby encouraging lower sum-of-costs.

Notably, \cref{eq:reward} contains no collision penalty for either obstacles or other robots.
Incorporating such penalties into the reward can dilute the goal-directed learning signal; instead, we address the \textbf{safeguarding gap} through a specialised rollout scheme for double-integrator dynamics that prevents collisions algorithmically.

\subsection{Safeguarded Rollout with Backup Trajectories}
\label{sec:safe-rollout}
Recall that the policy outputs an action distribution $\pi(o_i)$ for robot $i$ over the set of motion primitives $M$, which by itself does not determine the action to execute.
A straightforward strategy is to sample a primitive $m \in M$ according to $\pi(o_i)$; however, the resulting rollout under \cref{eq:dynamics} is not necessarily collision-free.
Rather than penalising unsafe primitives during learning, we develop a provably safe rollout scheme by exploiting backup trajectories, analogous in spirit to control barrier functions for continuous control~\cite{ames2019control}.

\subsubsection{Safe Execution with Static Obstacles}
We first illustrate the basic idea for robot $i$ in isolation.
Given a state $\bm{x}_i$ and a motion primitive $m \in M$, we define a \emph{stopping trajectory} $\tau\sub{stop}(\bm{x}_i,m)$ by first applying $m$ to $\bm{x}_i$ and subsequently applying maximum acceleration opposite to the velocity until the robot comes to a complete stop.
This results in a state--control trajectory
$[\langle \bm{x}_i^0 ({=}\bm{x}_i), m \rangle,
\langle \bm{x}_i^{1}, -\bm{v} / |\bm{v}| \cdot \min\{a\sub{max}, \|\bm{v}\| / \Delta t\} \rangle,
\ldots, \langle \bm{x}\sub{stop}, \bm{0} \rangle]$,
where $\bm{v}$ represents the velocity component of $x_i$ and $x\sub{stop}$ has zero velocity.
We filter out motion primitives whose corresponding stopping trajectories collide with obstacles, yielding a \emph{safe primitive set} $M\suf{safe}(\bm{x}_i) \subseteq M$, abbreviated as $M\suf{safe}_i$ when the state $\bm{x}_i$ is clear from context.
This set is safe in the sense that, for any primitive selected from $M\suf{safe}_i$, a collision-free trajectory that brings the robot to a complete stop is always available.

Now, suppose that at timestep $t-1$, robot $i$ at state $\bm{x}_i^{t-1}$ executes a primitive $m \in M\suf{safe}(\bm{x}_i^{t-1})$, resulting in the next state $\bm{x}_i^t$.
If $M\suf{safe}(\bm{x}_i^t)$ is non-empty, selecting a primitive from this set again ensures the availability of a collision-free stopping trajectory.
If $M\suf{safe}(\bm{x}_i^t)$ is empty, however, the robot can continue executing the remaining portion of $\tau\sub{stop}(\bm{x}_i^{t-1},m)$, which is guaranteed to bring the robot to a stop without collision.
We refer to this stopping trajectory retained from the previously selected primitive as the \emph{backup trajectory}.
This rollout strategy, which switches to the backup trajectory if no safe primitive is available, prevents the robot from entering a dead-end state, a state from which no collision-free action sequence exists.
Given that $\bm{x}_i^0$ is collision-free and has zero velocity, simple induction shows that this execution scheme guarantees collision-free motion with respect to obstacles.

\subsubsection{Safe Execution with Other Robots}
We now leverage the backup trajectory scheme for safe multi-robot execution by integrating it with \emph{PIBT collision shielding}~\cite{okumura2022priority,veerapaneni2024improving}, a successful paradigm for substantially improving the performance of neural MAPF policies.
In the context of grid-world MAPF, PIBT receives an agent-wise ordered list of one-step actions, called a \emph{preference}, and greedily assigns actions to construct a collision-free configuration from the current one.
Such preferences can be constructed, for example, by ordering actions according to their cost-to-go values to individual goals or directly from the outputs of neural policies.
PIBT guarantees that a feasible action assignment can always be found, with this feasibility relying on the availability of the `stay' action as a safe fallback when preferred actions cannot be assigned, provided that interacting agents can likewise fall back to their stay actions.

In kinodynamic settings, however, such a stay action is not always available due to non-zero velocity.
PIBT has been extended to general motion primitives~\cite{moldagalieva2026dblacam};
while this variant guarantees collision-free outcomes, it does not guarantee that a feasible action assignment can always be found, which is particularly critical in dense scenarios.
Our key observation is that, under double-integrator dynamics, the backup trajectories introduced above can play a role analogous to the `stay' action, thereby providing both safety, i.e., collision-free execution, and feasibility, i.e., a feasible action assignment can always be found in finite time.

{
  \newcommand{\valid}{\n{\texttt{valid}}}
  \newcommand{\invalid}{\n{\texttt{invalid}}}
  \begin{algorithm}[t!]
    \caption{Safeguarded rollout with PIBT and backups}
    \label{algo:pibt}
    \begin{algorithmic}[1]
      \small
      \Input{config $(\bm{x}_1, \ldots, \bm{x}_n)$, backup actions $(b_1, \ldots, b_n)$}
      \Output{action $m_i$ for each robot $i\in A$, initialised with $\bot$}
      \State \textbf{for}~$i \in A$~\textbf{do}~%
      \textbf{if}~$m_i=\bot$~\textbf{then}~\Call{PIBT}{$i$}
      \label{algo:pibt:top-level}
      \smallskip
      \Procedure{PIBT}{$i$}
      \label{algo:pibt:recursive}
      \State sort $M\suf{safe}_i$
      \Comment{based on $\pi$ in \gambit}
      \label{algo:pibt:sort}
      \For{$m \in M\suf{safe}_i$}
      \IfSingle{$\funcname{collide\_stop}(i, m)$}{\textbf{continue}}
      \State $m_i \gets m$
      \While{
        $\exists j\in A, m_j=\bot, \funcname{collide\_stop}(j, b_j)$
      }
      \label{algo:pibt:pi-check}
      \IfSingle{$\Call{PIBT}{j} = \invalid$}{\textbf{break}}
      \EndWhile
      \IfSingle{no \invalid received}{\Return \valid}
      \EndFor
      \State $m_i \gets b_i$
      \label{algo:pibt:backup}
      \State \Return \invalid
      \EndProcedure
    \end{algorithmic}
  \end{algorithm}
}

\Cref{algo:pibt} provides a concrete procedure for this safeguarded rollout, which takes as input the current state $\bm{x}_i$ and the first action $b_i$ of the backup trajectory for each robot~$i$.
Its overall structure follows the original PIBT algorithm, where a recursive subprocedure $\mathsf{PIBT}$ (\cref{algo:pibt:recursive}--) attempts to assign an action greedily and is invoked sequentially for each agent (\cref{algo:pibt:top-level}).
The preference over actions is determined at \cref{algo:pibt:sort} and can in principle be defined arbitrarily;
in this study, we stochastically sample the order of motion primitives, based on the probabilities assigned by the policy $\pi$.
The subprocedure returns \texttt{valid} if robot $i$ successfully secures a non-colliding action;
otherwise, it returns \texttt{invalid}, forcing the interacting agent to replan.
The function $\funcname{collide\_stop}(i, m)$ determines whether robot $i$, by executing primitive $m$ from its current state $\bm{x}_i$, would collide with robots whose actions have already been assigned.
Our key modification to the original PIBT is that this collision check extends beyond the immediate execution of $m$ to the corresponding stopping trajectories;
this ensures that robots can remain collision-free in the future, provided that they execute their corresponding backup trajectories.
\Cref{algo:pibt} falls back to the backup trajectory if no primitive can be safely assigned (\cref{algo:pibt:backup}).

\begin{theorem}
In finite time, \cref{algo:pibt} finds a safe action assignment together with collision-free subsequent stopping trajectories, provided that the input backup trajectories are jointly collision-free.
\label{thrm:one-step}
\end{theorem}
\begin{proof}\emph{(sketch)}
The collision-free property follows from the analysis of the original PIBT and its motion-primitive-based variant~\cite{okumura2022priority,moldagalieva2026dblacam}.
In short, the subprocedure $\mathsf{PIBT}$ is invoked only for a robot whose action has not yet been assigned, and each robot finalises its action at most once.
A robot $i$ finalises its action $m_i$ only if the corresponding stopping trajectory does not collide with those of robots whose actions have already been assigned.
We guarantee that unassigned robots have a safe action available by ensuring that robot $i$'s actions do not conflict with the backup trajectories of these unassigned robots;
otherwise, $\mathsf{PIBT}$ is recursively invoked for the interacting robot $j$ to resolve the conflict before the action of $i$ is finalised.
If no safe primitive can be assigned through this recursion, the involved robots fall back to their input backup trajectories, which are collision-free by assumption.
Hence, the recursion terminates after finitely many assignments while preserving collision-free stopping trajectories for all robots.
\end{proof}

\begin{theorem}
  Consider a multi-robot execution in which \cref{algo:pibt} is invoked every $\Delta t$, the resulting actions are executed, and the corresponding stopping trajectories are retained as backup trajectories for the next invocation.
  Then, the resulting multi-robot trajectories remain collision-free for all time.
\end{theorem}
\begin{proof}
By the previous theorem, each PIBT invocation preserves collision-freeness and produces collision-free stopping trajectories, which are used as the backup trajectories for the next invocation.
The claim therefore follows by induction over timesteps.
\end{proof}

\subsubsection{Auxiliary Reward for RL}
With the safeguarded rollout, an ordered list of motion primitives is sampled from the policy output $\pi(o_i)$, after which the selected actions may be overridden by PIBT collision shielding.
While safety is ensured, such interventions can weaken the learning signal for the policy itself to synthesise coordinated actions.
Therefore, we augment the reward defined in \cref{eq:reward} with an auxiliary term based on whether the policy's actions are consistent with the PIBT-processed actions, penalising the overriding of the actions by the PIBT collision shielding.

{
\begin{figure*}[t]
\vspace*{1pt}
\centering
\begin{tikzpicture}
  \small
  \newcommand{\figwidth}{0.16\linewidth}
  \newcommand{\figheight}{0.12\linewidth}
  \newcommand{\instanceSize}{0.145\linewidth}
  \newcommand{\entry}[4]{
    \node[anchor=south] at (#1 * \figwidth + 0.56*\figwidth, 1.06*\instanceSize) {\scriptsize \mapname{#3}};
    \node[anchor=south] at (#1 * \figwidth + 0.56*\figwidth, 0.98*\instanceSize) {\tiny ($n{=}#4$)};
    \node[anchor=north east] at (#1 * \figwidth + 1.06*\figwidth, 1.04*\instanceSize)
         {\includegraphics[width=\instanceSize,height=\instanceSize]{fig/raw/instances/#2_#4agents}};
    \node[anchor=north west] at (#1 * \figwidth, 0)
         {\includegraphics[width=\figwidth,height=\figheight]{fig/raw/plots/#2_success_rate}};
    \node[anchor=north west] at (#1 * \figwidth, -1.0*\figheight)
         {\includegraphics[width=\figwidth,height=\figheight]{fig/raw/plots/#2_sum_of_costs}};
    \node[anchor=north west] at (#1 * \figwidth, -2.0*\figheight)
         {\includegraphics[width=\figwidth,height=\figheight]{fig/raw/plots/#2_time}};
  }
  \entry{0}{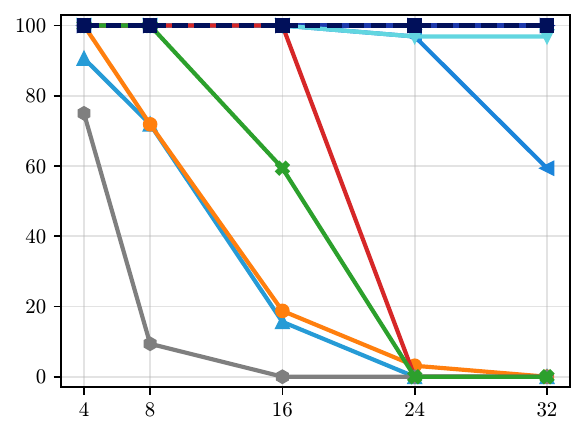}{Empty}{32}
  \entry{1}{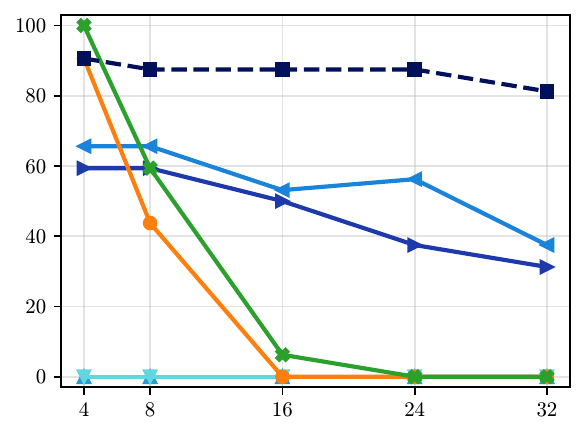}{Blocked Random}{8}
  \entry{2}{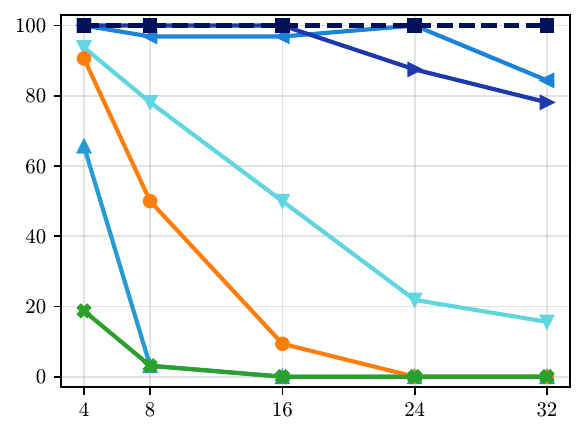}{Forest}{8}
  \entry{3}{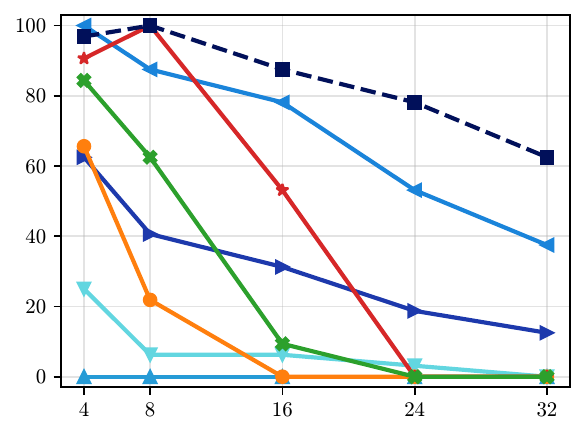}{Conveyor}{32}
  \entry{4}{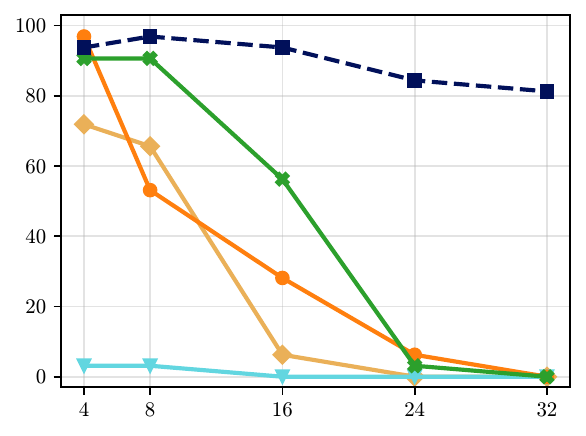}{Blocked Random-3D}{8}
  \entry{5}{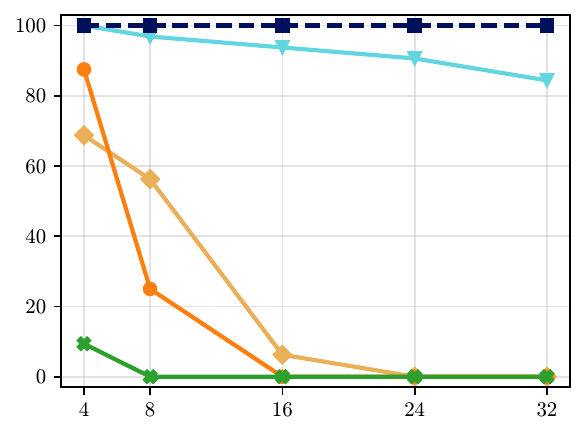}{Forest-3D}{16}
  %
  \newcommand{\socabove}[2]{
    \definecolor{dblacam}{RGB}{255 127 14}
    \fill[dblacam] (#1 * \figwidth + #2*\figwidth, -1.03*\figheight)
      circle [radius=0.7pt];
    \draw[dblacam, line width=0.7pt]
      (#1 * \figwidth + #2*\figwidth, -1.09*\figheight)
      --++(0, 0.04*\figheight);
    \fill[dblacam]
      (#1 * \figwidth + #2*\figwidth, -1.05*\figheight)
      --++(-0.018*\figwidth, -0.018*\figheight)
      --++(0.036*\figwidth, 0) -- cycle;
  }
  \socabove{0}{0.175}
  \socabove{1}{0.15}
  \socabove{3}{0.15}
  \socabove{4}{0.195}
  \socabove{5}{0.175}
  %
  \begin{scope}[shift={(0,0)}]
    \node[anchor=south,rotate=90] at (0, -0.5*\figheight) {CSR [\%]~$\rightarrow$};
    \node[anchor=south,rotate=90] at (0, -1.5*\figheight) {$\leftarrow$~SOC};
    \node[anchor=south,rotate=90] at (0, -2.5*\figheight) {$\leftarrow$~time [\si{\second}]};
  \end{scope}
  %
  \node[] at (0.55*\figwidth, -3.2*\figheight) {\# robots};
  %
  \begin{scope}[shift={(1.75*\figwidth, -3.15*\figheight)}]
    \newcommand{\mylegend}[7][solid]{
      \definecolor{c}{RGB}{#4 #5 #6}
      \draw[draw=c,thick,#1] (#2\linewidth, 0) --++(0.02\linewidth, 0)
      node[pos=0.5]{\textcolor{c}{#7}}
      node[right,pos=1]{\scriptsize #3};
    }
    \mylegend{0}{\solver{GAMBIT}}{1}{16}{90}{\scalebox{0.6}{$\blacksquare$}}
    \mylegend{0.13}{\solver{GAMBIT $\backslash$RL}}{29}{57}{172}{\scalebox{0.75}{$\blacktriangleright$}}
    \mylegend{0.26}{\solver{GAMBIT $\backslash$IL}}{26}{132}{218}{\scalebox{0.75}{$\blacktriangleleft$}}
    \mylegend{0.39}{\solver{GAMBIT $\backslash$PIBT}}{38}{155}{214}{\scalebox{0.75}{$\blacktriangle$}}
    \mylegend{0.52}{\solver{GAMBIT $\backslash$IL $\backslash$RL}}{98}{214}{224}{\scalebox{0.75}{$\blacktriangledown$}}
  \end{scope}
  \begin{scope}[shift={(1.75*\figwidth, -3.35*\figheight)}]
    \newcommand{\mylegend}[6]{
      \definecolor{c}{RGB}{#3 #4 #5}
      \draw[draw=c,thick] (#1\linewidth, 0) --++(0.02\linewidth, 0)
      node[pos=0.5]{\textcolor{c}{#6}}
      node[right,pos=1]{\scriptsize #2};
    }
    \newcommand{\filledhexagon}{%
      \tikz[baseline=-0.6ex, x=0.5ex, y=0.5ex]{
        \fill[c] (1,0) -- (0.5,0.866) -- (-0.5,0.866) -- (-1,0)
          -- (-0.5,-0.866) -- (0.5,-0.866) -- cycle;
      }%
    }
    \newcommand{\filleddiamond}{%
      \tikz[baseline=-0.6ex, x=0.6ex, y=0.6ex]{
        \fill[c] (0,1) -- (1,0) -- (0,-1) -- (-1,0) -- cycle;
      }%
    }
    \mylegend{0}{\solver{db-LaCAM}}{255}{127}{14}{\raisebox{-1.05ex}{$\bullet$}}
    \mylegend{0.13}{\solver{D4ORM}}{44}{160}{44}{\scalebox{0.75}{$\pmb{\pmb{\pmb{\times}}}$}}
    \mylegend{0.26}{\solver{MMD-XECBS}}{214}{39}{40}{\scalebox{0.6}{$\bigstar$}}
    \mylegend{0.39}{\solver{RVO}}{127}{127}{127}{\raisebox{-1.2ex}{\filledhexagon}}
    \mylegend{0.52}{\solver{EGO-Swarm}}{234}{176}{88}{\filleddiamond}
  \end{scope}
\end{tikzpicture}
\caption{
  Main result.
  For each scenario, an example instance is shown at the top, where filled circles represent robots and unfilled circles connected by lines indicate their corresponding goals.
  \textcolor[RGB]{153	153	153}{Grey} circles, rectangles, and cylinders represent obstacles.
  The remaining plots show, from top to bottom, CSR (complete success rate), SOC (sum-of-costs; normalised by lower bounds based on start--goal distances), and computation time to derive the entire trajectories.
  SOC and computation time are averaged over successful instances for each planner, with shaded regions indicating one standard deviation;
  to ensure meaningful comparisons, we report these metrics only when the corresponding success rate exceeds 50\%.
  db-LaCAM's SOC often exceeds the displayed range; we use an orange dot and arrow to indicate the clipped value to preserve readability of the remaining curves.
  \vspace*{-1ex}
}
\label{fig:result}
\end{figure*}
}

\subsection{Implementation Details}
\label{sec:training-details}
In 2D, the motion-primitive set $M$ comprises 8 acceleration vectors uniformly distributed across a single acceleration magnitude, with $a\sub{max}{=}\SI{2}{\meter\per\second\squared}$, together with zero acceleration, yielding 9 primitives in total.
Similarly, in 3D, we use 26 directions at the max acceleration magnitude, with $a\sub{max}{=}\SI{5}{\meter\per\second\squared}$, together with zero acceleration, yielding 27 primitives in total.
The robot radius $R\sub{robot}$ is set to \SI{0.25}{\meter}, the communication radius $R\sub{comm}$ to \SI{3}{\meter}, and the FOV $N\sub{obs}$ to 11 cells along each dimension, with a cell size of \SI{0.25}{\meter}.
The workspace is typically $\SI{6}{\meter}$ along each dimension and $\Delta t$ is taken to be \SI{0.1}{\second}.

For IL, we use 60K demonstrations, with half generated on \mapname{Blocked Random} instances and the other half on \mapname{Forest} instances, containing four agents.
Data collection requires approximately 22 hours of wall-clock time on Intel Xeon Platinum 8452Y CPUs, with 24 runs executed in parallel.
The policy contains approximately 1.4M trainable parameters and is trained using the AdamW optimiser~\cite{Loshchilov-AdamW}.
IL training requires approximately 40 GPU hours on NVIDIA L40S GPUs.
The RL phase requires an additional 25 GPU hours, with rollouts using the PIBT-based safeguarded rollout described above, with the reward parameters in \cref{eq:reward} set to $w_1{=}5.0$ and $w_2{=}0.01$.
The policy is trained in Python and subsequently exported for C++ inference to reduce runtime overhead.
The 3D setup is configured similarly.

\section{Experiments}
We evaluate \gambit as a reactive planner that repeatedly selects motion primitives from local observations, combined with the safeguarded rollout (\cref{algo:pibt}).
Visualisations of both the simulation and real-robot deployment are provided in the supplementary video.
All methods are evaluated using the compute infrastructure described in \cref{sec:training-details}.

\subsection{Simulation Setup}
\subsubsection{Instances}
We prepare 32 instances with varying numbers of robots from $4$ to $32$, based on the following scenarios:
\emph{(i)}~\mapname{Empty} (2D), representing a basic scenario without obstacles.
Start and goal positions are sampled uniformly at random;
\emph{(ii)}~\mapname{Blocked Random} (2D, 3D) and \emph{(iii)}~\mapname{Forest} (2D, 3D) described in \cref{sec:pretraining}, representing in-distribution scenarios.
The test instances are distinct from those used during training;
\emph{(iv)}~\mapname{Conveyor} (2D), adapted from~\cite{shaoul2025multi}, representing out-of-distribution scenarios.

{
\DeclareRobustCommand\bluestar{\begin{tikzpicture}
  \node[circle, inner sep=0.5pt, minimum size=.1cm, color=white, fill=blue]{$\star$}; 
\end{tikzpicture}}
\newcommand{\imgsize}{0.078\linewidth}
\newcommand{\trimleft}{393}
\newcommand{\trimtop}{85}
\newcommand{\trimsize}{25}
\newcommand{\entry}[6]{
  \tikzmath{
    \trimbottom = 424.8 - \trimtop - \trimsize;
    \trimright = 424.8 - \trimleft - \trimsize;
  }%
  \node[anchor=north west,
    draw,very thin,lightgray,
    inner sep=0,gray]
  at (#1 * \imgsize, -#2 * \imgsize)
       {\includegraphics[
           trim={\trimleft bp \trimbottom bp \trimright bp \trimtop bp},
           clip,width=\imgsize,height=\imgsize]
         {fig/raw/forest40_1024_slides/trajectory_#3_#4}};
   \node[] at (#1 * \imgsize + #5 * \imgsize, -#2 * \imgsize - #6 * \imgsize) {\textcolor{white}{$\star$}};
}
\begin{figure*}[t]
\vspace*{1pt}
\centering
\begin{tikzpicture}
  \begin{scope}
  \node[anchor=north west] at (0, 0)
       {\includegraphics[width=0.329\linewidth,height=0.329\linewidth,clip,trim={7bp 7bp 7bp 7bp}]
         {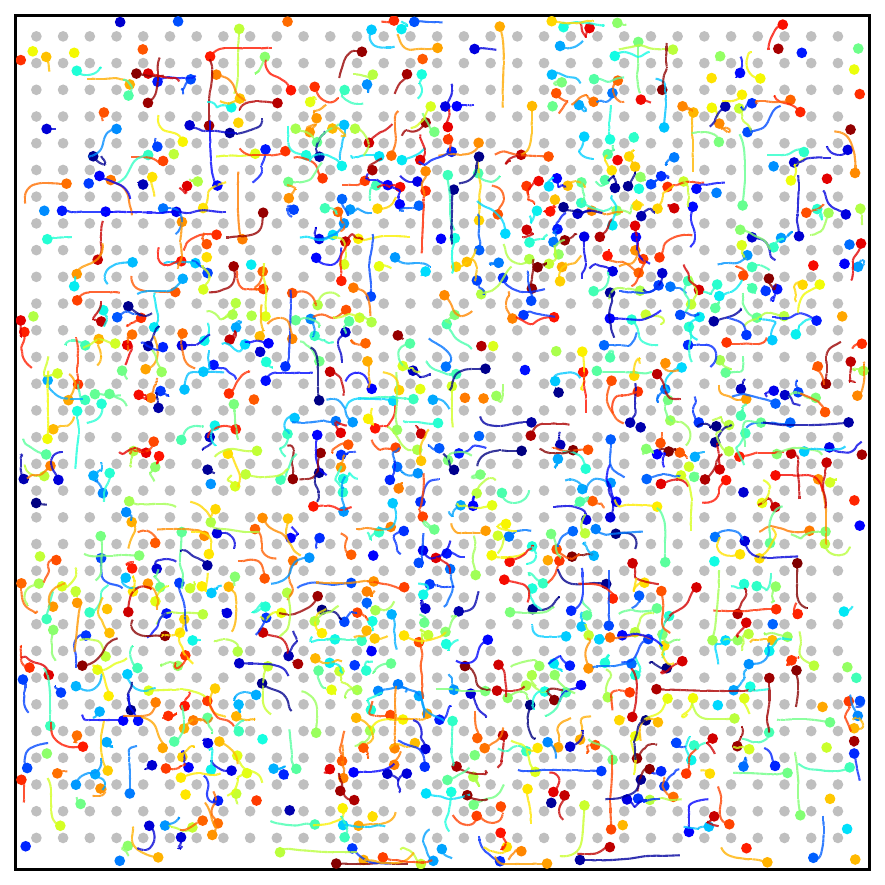}};
  \draw[color=black,line width=2.5pt] (0.317\linewidth, -0.09\linewidth) rectangle ++(0.02\linewidth, 0.02\linewidth);
  \draw[densely dotted](0.319\linewidth, -0.07\linewidth) --++ (0.05\linewidth, 0.06\linewidth);
  \draw[densely dotted](0.33\linewidth, -0.09\linewidth) --++ (0.04\linewidth, 0.003\linewidth);
  \node[] at (0.005\linewidth, -0.345\linewidth) {\tiny 0};
  \node[] at (0.34\linewidth, -0.345\linewidth) {\tiny \SI{40}{\meter}};
  \end{scope}
  \begin{scope}[shift={(0.36\linewidth, -0.007\linewidth)}]
  \entry{0}{0}{0060}{0080}{0.52}{0.71}
  \entry{1}{0}{0070}{0090}{0.5}{0.7}
  \entry{2}{0}{0080}{0100}{0.48}{0.7}
  \entry{3}{0}{0090}{0110}{0.35}{0.67}
  \entry{4}{0}{0100}{0120}{0.49}{0.68}
  \entry{5}{0}{0110}{0130}{0.51}{0.7}
  \entry{6}{0}{0120}{0140}{0.52}{0.7}
  \entry{7}{0}{0130}{0150}{0.51}{0.7}
  \draw[thick,->] (0, -1.15 * \imgsize) --++ (7.4*\imgsize, 0) node[right,pos=1]{time};
  \end{scope}
  \begin{scope}[shift={(0.36\linewidth, -0.115\linewidth)}]
    \newcommand{\figwidth}{0.15\linewidth}
    \newcommand{\figheight}{0.2\linewidth}
    \newcommand{\evoplot}[2]{
      \node[anchor=north west] at (#1 * \figwidth, 0)
           {\includegraphics[width=\figwidth,height=\figheight]{fig/raw/success_rate_evolution/#2}};
      \node[anchor=south] at (#1 * \figwidth + 0.55*\figwidth, -0.018\linewidth) {\scriptsize $n{=}#2$};
    }
    \evoplot{0}{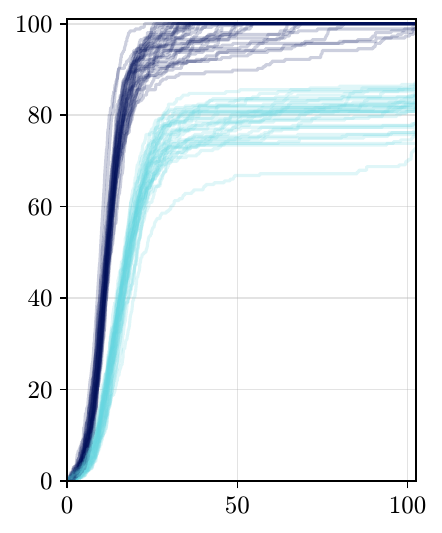}
    \evoplot{1}{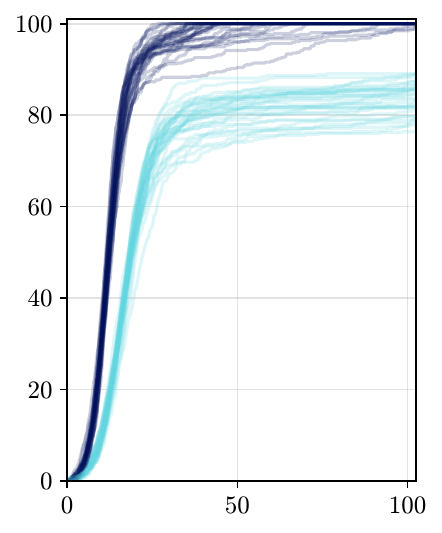}
    \evoplot{2}{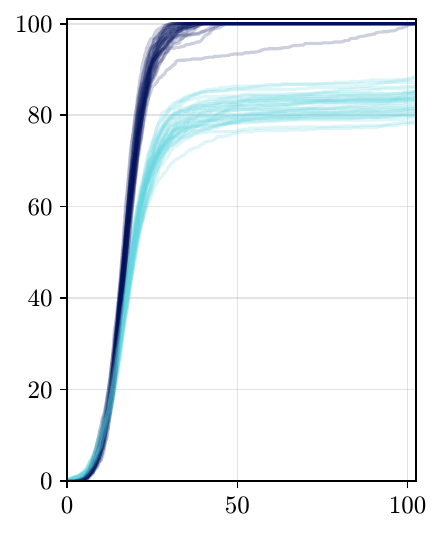}
    \node[rotate=90] at (0.005\linewidth, -0.5*\figheight) {\scriptsize ISR [\%]};
    \node[] at (0.57*\figwidth, -1.06 * \figheight) {\scriptsize simulation time [s]};
    \definecolor{gambitblue}{RGB}{9	19	87}
    \definecolor{gambitred}{RGB}{84	154	209}
    \newcommand{\legendentry}[3]{
      \draw[draw=#1,thick] (#2\linewidth, -1.06*\figheight) --++(0.02\linewidth, 0)
      node[right,pos=1]{\scriptsize #3};
    }
    \legendentry{gambitblue}{0.2}{\solver{GAMBIT}}
    \legendentry{gambitred}{0.3}{\solver{GAMBIT $\setminus$IL $\setminus$RL}}
    \node[anchor=north west] at (3.15 * \figwidth, -0.05)
         {\includegraphics[width=0.14\linewidth,height=0.188\linewidth,clip,trim={32bp 0 0 0},
           ]{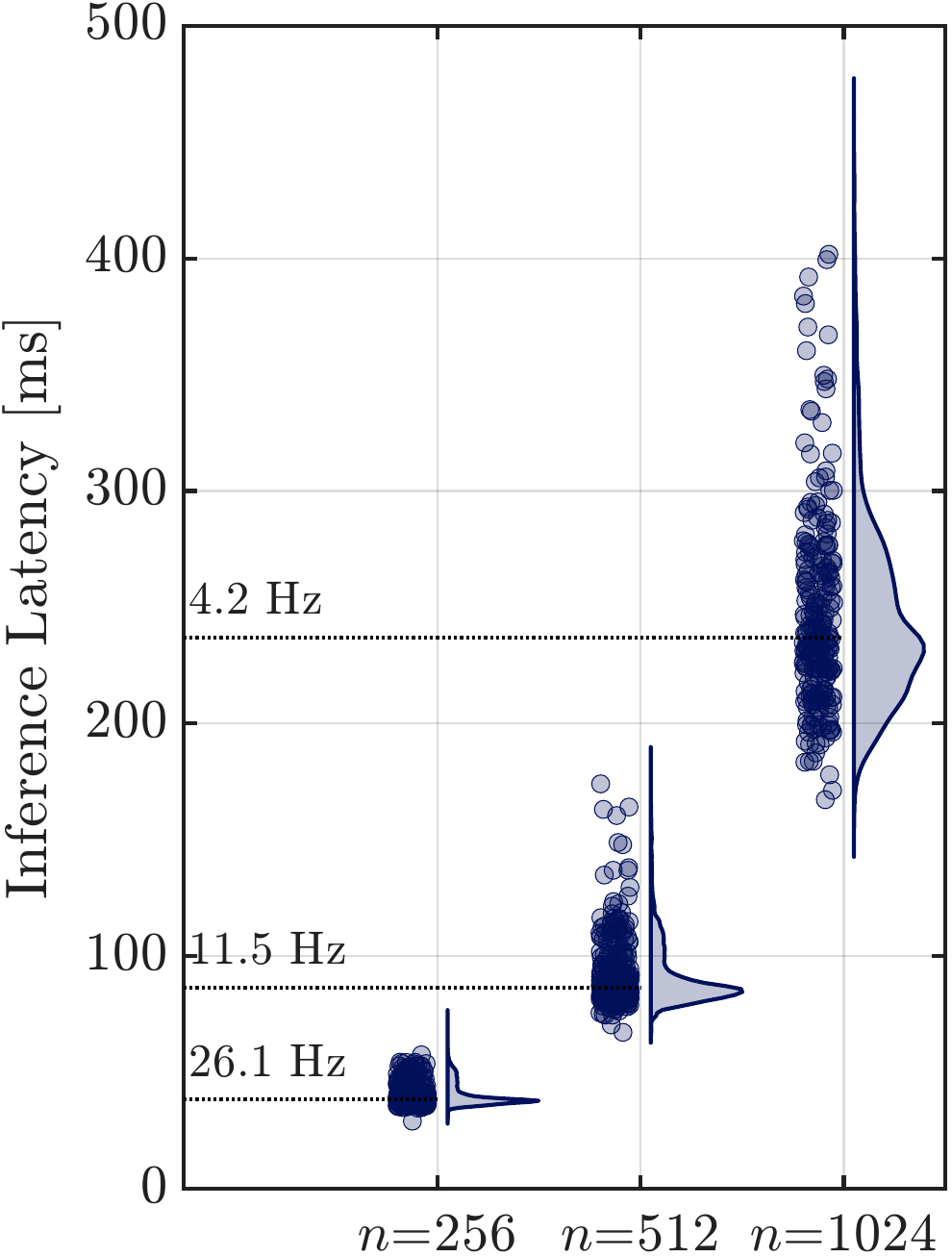}};
    \node[rotate=90] at (3.15 * \figwidth, -0.09\linewidth) {\scriptsize inference latency [\si{\milli\second}]};
    \node[] at (3.7*\figwidth, -1.06 * \figheight) {\scriptsize \#robots};
  \end{scope}

\end{tikzpicture}
\caption{
  Scalability assessment with up to $1{,}024$ robots in the \mapname{Forest} scenario.
  The left panel shows a snapshot of the entire workspace at a certain timestep.
  Circles represent robots, with solid lines indicating their trajectories as tails.
  The upper-right panel depicts the trajectory evolution within a specific region, highlighting coordinated behaviours produced by \gambit; e.g., the
  \bluestar{} robot
  temporarily moves away from its goal to let others pass.
  The lower-right panel presents individual success rate (ISR) evolution for $n=\{256, 512, 1{,}024\}$;
  i.e., the percentage of robots that have reached their goals, as a function of elapsed simulation time.
  We run the planners on 32 instances, with each curve representing the solution over a different instance.
  The bottom-right panel shows the distribution around the median inference time/step (and the corresponding loop frequency) for computing actions for the entire team.
}
\label{fig:scalability}
\end{figure*}
}

\subsubsection{Baselines}
On these instances, we contrast \gambit against the following representative baselines spanning different planning and optimisation paradigms:
\begin{itemize}
\item \solver{RVO}~\cite{van2011reciprocal}, a widely used reactive collision-avoidance algorithm.
  We use its double-integrator formulation~\cite{van2011avo}.
\item \solver{db-LaCAM}~\cite{moldagalieva2026dblacam}, a centralised, full-horizon kinodynamic multi-robot motion planner that computes complete trajectories from start to goal.
  A variant of this planner is also used to generate the dataset for IL.
\item \solver{MMD-XECBS}~\cite{shaoul2025multi},
  a centralised, full-horizon planner based on diffusion models trained for the target maps.
  We include this method as a learning-based alternative.
\item \solver{D4ORM}~\cite{zhang2025d4orm}, a GPU-accelerated population-based trajectory smoother for multi-robot systems.
\item \solver{EGO-Swarm}~\cite{zhou2021ego}, a decentralised reactive trajectory planner designed for 3D multi-drone navigation.
\end{itemize}
For all baselines, we use the publicly available codebases provided by the respective authors.
Due to their implementation constraints, \solver{RVO} supports only 2D obstacle-free scenarios (i.e. \mapname{Empty}), whereas \solver{EGO-Swarm} is evaluated only in 3D.
As \solver{MMD} requires training on the target scenario, we evaluate it only on \mapname{Empty} and \mapname{Conveyor}.
All planners are allocated a time limit of \SI{1}{\minute}.

We additionally evaluate the following variants of \gambit to assess the contribution of each technical component:
\begin{itemize}
\item \solver{GAMBIT $\setminus$IL $\setminus$RL} removes the learning components entirely, and instead sorts motion primitives in \cref{algo:pibt:sort} using a manually designed heuristic.
  This heuristic considers the positional distance to the goal from the rollout state together with the resulting velocity, representing greedy behaviour towards the goal.
\item \solver{$\setminus$RL} ablates the RL fine-tuning stage, using the policy trained solely by IL.
\item \solver{$\setminus$IL} removes IL, with the policy trained purely by RL until convergence ($\approx\SI{65}{\hour}$).
\item \solver{$\setminus$PIBT} removes the safeguarded rollout, directly samples a motion primitive from the action distribution, and executes it according to \cref{eq:dynamics}.
  For this variant, we modify the RL reward to penalise collisions.
\end{itemize}

\subsubsection{Metrics}
The following metrics assess the performance of each planner:
\emph{complete success rate (\textbf{CSR})} measures the fraction of instances in which \emph{all} agents reach their goals;
\emph{sum-of-costs (\textbf{SOC})} assesses trajectory quality;
and \emph{computation time (\textbf{Time})} measures the wall-clock time required to compute the complete trajectories.
We additionally use \emph{individual success rate (\textbf{ISR})}, measuring the fraction of agents that reach their goals within each instance.

\subsection{Main Results}
\Cref{fig:result} summarises our main results.
\gambit achieves high success rates across all tested scenarios, including both in-distribution and out-of-distribution settings, at scale while maintaining highly competitive solution quality, delivering the strongest overall performance among the evaluated methods.
The computation time remains moderate; even with $32$ robots, \gambit requires only about \SI{1}{\second} to derive feasible trajectories, suggesting real-time planning ability, which we investigate further in \cref{sec:exp:scalability,sec:exp:demo}.
A closer look reveals that learning is a key component (see \solver{GAMBIT $\setminus$IL $\setminus$RL}), with both IL and RL providing substantial improvements in success rate.
The performance of \solver{$\setminus$PIBT} further demonstrates that the safeguarded rollout is critical to achieving strong performance; i.e., direct execution of the learned policy is not performant.
We conjecture that this is because multi-robot navigation policies generally need to learn multiple objectives jointly, including goal reaching and collision avoidance, whereas the safeguarded rollout largely relieves the policy of the latter, simplifying the learning problem and allowing the policy to focus its learning capacity on goal-directed coordination.

Among the baselines, the domain-independent planners \solver{db-LaCAM} and \solver{D4ORM} struggle to scale, while \solver{MMD-XECBS}, which uses learning to adapt to the target domain, achieves superior planning performance but at a high computational cost.
Decentralised local collision-avoidance planners such as \solver{RVO} and \solver{EGO-Swarm} struggle with interaction-heavy, dense instances where many robots are packed into a limited workspace, as these scenarios require coordination beyond mere collision avoidance.
While remaining reactive at execution time, \gambit incorporates such long-horizon reasoning through offline learning.

While \gambit is capable of dense coordination, some failures remain.
Our qualitative observation suggests that these failures are primarily caused by livelocks, where one or more robots repeatedly execute the same actions without making progress.
Livelocks occur primarily near obstacles, e.g., when a robot must pass through a narrow gap between obstacles but safety margins prevent the motion, or when it must traverse around an obstacle to reach a goal on the opposite side.
We envision two promising directions for overcoming this limitation: stronger policies through improved architectures and training pipelines, and integrating the learned motion-primitive heuristic into long-horizon search~\cite{zhang2025d4orm,moldagalieva2026dblacam}.
Importantly, none of the observed failures are caused by collisions, owing to the safeguarded rollout.

\subsection{Scalability Assessment}
\label{sec:exp:scalability}
Using the \mapname{Forest} scenario with an enlarged workspace, we evaluate the scalability of \gambit on instances with up to $1{,}024$ robots.
At this scale, none of the baseline methods successfully solve the instances, and thus we evaluate \gambit and its non-learning variant (i.e., \solver{$\setminus$IL $\setminus$RL}).
\Cref{fig:scalability} demonstrates that \gambit provides a highly scalable planning framework, maintaining high individual success rates even in densely packed settings.
This scale is significantly beyond the reach of existing methods for computing full-horizon continuous multi-robot trajectories.
Moreover, \gambit retains real-time responsiveness at scale: with $512$ robots, it synthesises safe coordinated actions in $<\SI{100}{\milli\second/step}$, enabling real-time control at scale.
While the non-learning, greedy variant still allows a portion of the robots to reach their goals (see the ISR evolution plot), some robots remain unsuccessful due to its coordination-unaware heuristic, emphasising the importance of learning coordination-aware heuristics.

{%
\setlength{\fboxsep}{0pt}%
\setlength{\fboxrule}{0.1px}%
\begin{figure}[t!]
    \vspace*{6pt}
    \centering
        \includegraphics[width=0.95\linewidth]
            {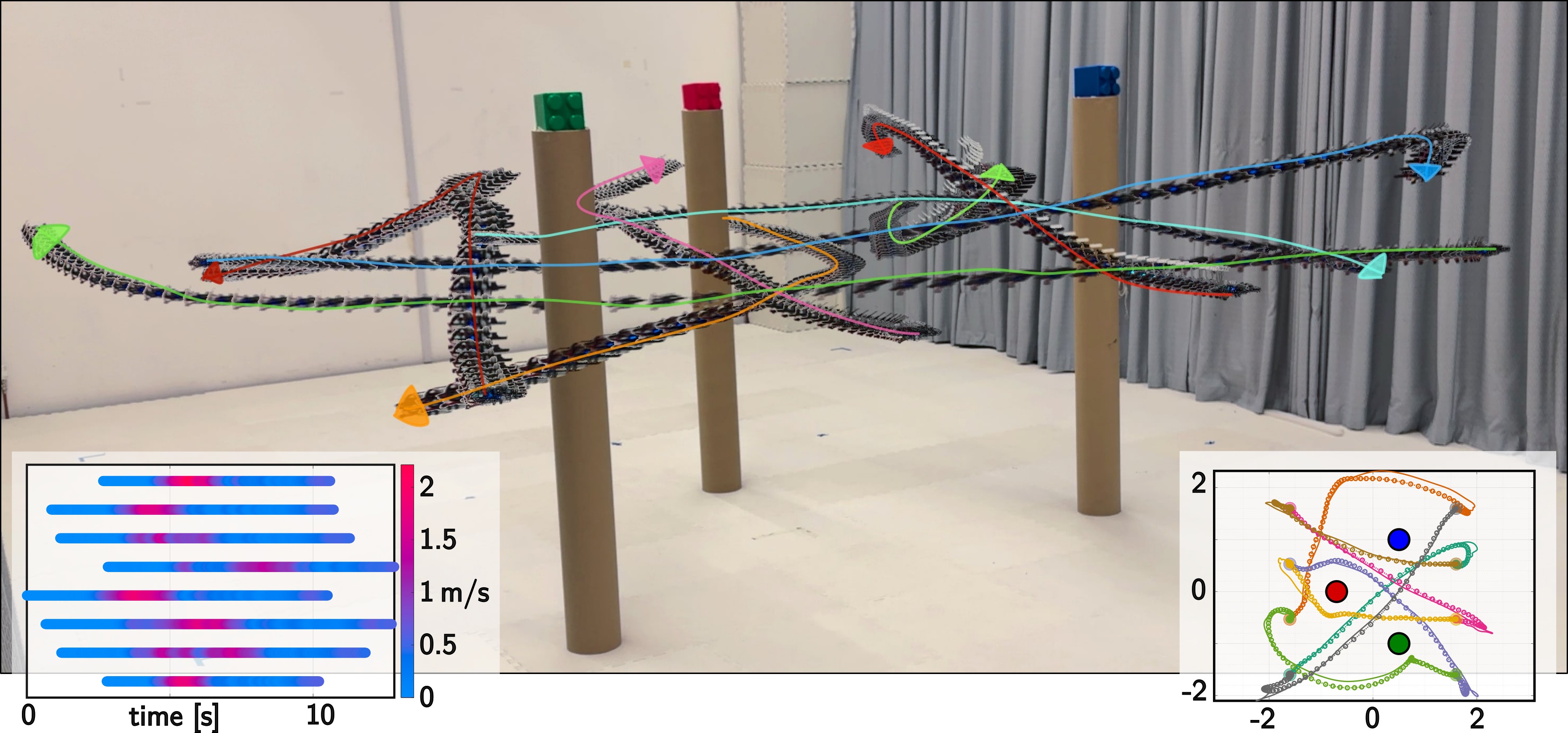}%
\caption{Snapshot from a \gambit deployment with eight multirotors planning at \SI{10}{Hz}.
\textbf{(left-inset)} Speed profile $\mathrm{(m/s)}$ for the robots in this instance.
\textbf{(right-inset)} Top-down view of the arena and trajectories.
\vspace*{-2ex}
}
\label{fig:drones}
\end{figure}
}

\subsection{Multirotor Deployment}
\label{sec:exp:demo}
We finally deploy \gambit on physical multirotors~\cite{woo2025sanity}.
\Cref{fig:drones} shows overlaid trajectories from one demonstration in which eight multirotors swap their start--goal locations in a motion-capture arena containing obstacles.
We run \gambit at \SI{10}{\hertz} to synthesise motion online, yielding collision-free execution under closed-loop control even at a peak speed of \SI{2.2}{m/s} in the instance shown.

\section{Conclusion}
We presented \gambit, a framework for learning motion-primitive heuristics for dense multi-robot coordination.
While the resulting policy combined with our safeguarded rollout already achieves strong planning performance and exceptional scalability, we view \gambit more broadly as a powerful toolkit for building more capable multi-robot planners.
Indeed, learned components have increasingly demonstrated their potential to enhance multi-agent planning~\cite{alkazzi2024a,shaoul2025multi,jain2026graph}.
Promising directions include extending it to more general dynamics beyond the double-integrator model.

\bibliographystyle{IEEEtran}
\bibliography{sty/ref-macro,bibliography}  

@book{lavalle2006planning,
  title         = {Planning algorithms},
  author        = {LaValle, Steven Michael},
  year          = {2006},
  publisher     = {Cambridge university press}
}

@article{likhachev2009planning,
  title         = {Planning long dynamically feasible maneuvers for autonomous vehicles},
  author        = {Likhachev, Maxim and Ferguson, Dave},
  journal       = ijrr,
  year          = {2009}
}

@inproceedings{pivtoraiko2011kinodynamic,
  title         = {Kinodynamic motion planning with state lattice motion primitives},
  author        = {Pivtoraiko, Mihail and Kelly, Alonzo},
  booktitle     = iros,
  year          = {2011}
}

@inproceedings{van2011reciprocal,
  title         = {Reciprocal n-body collision avoidance},
  author        = {Van Den Berg, Jur and Guy, Stephen J and Lin, Ming and Manocha, Dinesh},
  booktitle     = isrr,
  year          = {2011}
}

@inproceedings{van2011avo,
  title         = {Reciprocal collision avoidance with acceleration-velocity obstacles},
  author        = {Van Den Berg, Jur and Snape, Jamie and Guy, Stephen J and Manocha, Dinesh},
  booktitle     = icra,
  year          = {2011}
}

@inproceedings{Vaswani-Attention,
  author        = {Vaswani, Ashish and Shazeer, Noam and Parmar, Niki and Uszkoreit, Jakob and Jones, Llion and Gomez, Aidan N and Kaiser, Lukasz and Polosukhin, Illia},
  title         = {Attention Is All You Need},
  booktitle     = neurips,
  year          = {2017}
}

@inproceedings{ames2019control,
  title         = {Control barrier functions: Theory and applications},
  author        = {Ames, Aaron D and Coogan, Samuel and Egerstedt, Magnus and Notomista, Gennaro and Sreenath, Koushil and Tabuada, Paulo},
  booktitle     = ecc,
  year          = {2019}
}

@article{camci2019learning,
  title         = {Learning motion primitives for planning swift maneuvers of quadrotor},
  author        = {Camci, Efe and Kayacan, Erdal},
  journal       = auro,
  year          = {2019}
}

@inproceedings{Loshchilov-AdamW,
  title         = {Decoupled weight decay regularization},
  author        = {Loshchilov, Ilya and Hutter, Frank},
  booktitle     = iclr,
  year          = {2019}
}

@article{sartoretti2019primal,
  title         = {Primal: Pathfinding via reinforcement and imitation multi-agent learning},
  author        = {Sartoretti, Guillaume and Kerr, Justin and Shi, Yunfei and Wagner, Glenn and Kumar, TK Satish and Koenig, Sven and Choset, Howie},
  journal       = ral,
  year          = {2019}
}

@article{solovey2020complexity,
  author        = {Kiril Solovey},
  title         = {Complexity of Planning},
  journal       = {arXiv preprint},
  year          = {2020}
}

@article{riviere2020glas,
  title         = {Glas: Global-to-local safe autonomy synthesis for multi-robot motion planning with end-to-end learning},
  author        = {Riviere, Benjamin and H{\"o}nig, Wolfgang and Yue, Yisong and Chung, Soon-Jo},
  journal       = ral,
  year          = {2020}
}

@article{li2021message,
  title         = {Message-aware graph attention networks for large-scale multi-robot path planning},
  author        = {Li, Qingbiao and Lin, Weizhe and Liu, Zhe and Prorok, Amanda},
  journal       = ral,
  year          = {2021}
}

@article{tordesillas2021mader,
  title         = {MADER: Trajectory planner in multiagent and dynamic environments},
  author        = {Tordesillas, Jesus and How, Jonathan P},
  journal       = tro,
  year          = {2021}
}

@inproceedings{zhou2021ego,
  title         = {Ego-swarm: A fully autonomous and decentralized quadrotor swarm system in cluttered environments},
  author        = {Zhou, Xin and Zhu, Jiangchao and Zhou, Hongyu and Xu, Chao and Gao, Fei},
  booktitle     = icra,
  year          = {2021}
}

@inproceedings{batra2022decentralized,
  title         = {Decentralized control of quadrotor swarms with end-to-end deep reinforcement learning},
  author        = {Batra, Sumeet and Huang, Zhehui and Petrenko, Aleksei and Kumar, Tushar and Molchanov, Artem and Sukhatme, Gaurav S},
  booktitle     = corl,
  year          = {2022}
}

@article{okumura2022priority,
  title         = {Priority Inheritance with Backtracking for Iterative Multi-agent Path Finding},
  journal       = aij,
  year          = {2022},
  author        = {Keisuke Okumura and Manao Machida and Xavier Défago and Yasumasa Tamura}
}

@inproceedings{ouyang2022training,
  title         = {Training language models to follow instructions with human feedback},
  author        = {Ouyang, Long and Wu, Jeffrey and others},
  booktitle     = neurips,
  year          = {2022}
}

@article{wen2022cl,
  title         = {CL-MAPF: Multi-agent path finding for car-like robots with kinematic and spatiotemporal constraints},
  author        = {Wen, Licheng and Liu, Yong and Li, Hongliang},
  journal       = {Robotics and Autonomous Systems},
  year          = {2022}
}

@inproceedings{yu2022surprising,
  title         = {The surprising effectiveness of ppo in cooperative multi-agent games},
  author        = {Yu, Chao and Velu, Akash and Vinitsky, Eugene and Gao, Jiaxuan and Wang, Yu and Bayen, Alexandre and Wu, Yi},
  booktitle     = neurips,
  year          = {2022}
}

@article{alkazzi2024a,
  title         = {A Comprehensive Review on Leveraging Machine Learning for Multi-Agent Path Finding},
  author        = {Alkazzi, Jean-Marc and Okumura, Keisuke},
  year          = {2024},
  journal       = {IEEE Access}
}

@inproceedings{okumura2024engineering,
  title         = {Engineering LaCAM$^\ast$: Towards Real-Time, Large-Scale, and Near-Optimal Multi-Agent Pathfinding},
  author        = {Okumura, Keisuke},
  year          = {2024},
  booktitle     = aamas
}

@article{ortiz2024idb,
  title         = {iDb-A*: Iterative search and optimization for optimal kinodynamic motion planning},
  author        = {Ortiz-Haro, Joaquim and H{\"o}nig, Wolfgang and Hartmann, Valentin N and Toussaint, Marc},
  journal       = tro,
  year          = {2024}
}

@inproceedings{veerapaneni2024improving,
  title         = {Improving Learnt Local MAPF Policies with Heuristic Search},
  author        = {Veerapaneni, Rishi and Wang, Qian and Ren, Kevin and Jakobsson, Arthur and Li, Jiaoyang and Likhachev, Maxim},
  booktitle     = icaps,
  year          = {2024}
}

@article{franke2025accelerating,
  title         = {Accelerating db-A* for kinodynamic motion planning using diffusion},
  author        = {Franke, Julius and Moldagalieva, Akmaral and Hanfeld, Pia and H{\"o}nig, Wolfgang},
  journal       = {arXiv preprint},
  year          = {2025}
}

@misc{woo2025sanity,
  title         = {Sanity: An Agile Brushless Quadrotor for Multi-Agent Experiments},
  author        = {Heedo Woo and Kazi R. I. Sanim and Keisuke Okumura and Guang Yang and Ajay Shankar and Amanda Prorok},
  year          = {2025},
  note          = {\url{https://openreview.net/forum?id=i5dKOANPZZ}}
}

@article{hou2025primitive,
  title         = {Primitive-Swarm: An Ultra-Lightweight and Scalable Planner for Large-Scale Aerial Swarms},
  author        = {Hou, Jialiang and Zhou, Xin and Pan, Neng and Li, Ang and Guan, Yuxiang and Xu, Chao and Gan, Zhongxue and Gao, Fei},
  journal       = tro,
  year          = {2025}
}

@inproceedings{jiang2025deploying,
  title         = {Deploying ten thousand robots: Scalable imitation learning for lifelong multi-agent path finding},
  author        = {Jiang, He and Wang, Yutong and Veerapaneni, Rishi and Duhan, Tanishq and Sartoretti, Guillaume and Li, Jiaoyang},
  booktitle     = icra,
  year          = {2025}
}

@article{moldagalieva2025db,
  title         = {db-ECBS: Interaction-Aware Multirobot Kinodynamic Motion Planning},
  author        = {Moldagalieva, Akmaral and Ortiz-Haro, Joaquim and H{\"o}nig, Wolfgang},
  journal       = tro,
  year          = {2025}
}

@inproceedings{shaoul2025multi,
  title         = {Multi-robot motion planning with diffusion models},
  author        = {Shaoul, Yorai and Mishani, Itamar and Vats, Shivam and Li, Jiaoyang and Likhachev, Maxim},
  booktitle     = iclr,
  year          = {2025}
}

@inproceedings{wang2025lns2+,
  title         = {LNS2+ RL: Combining multi-agent reinforcement learning with large neighborhood search in multi-agent path finding},
  author        = {Wang, Yutong and Duhan, Tanishq and Li, Jiaoyang and Sartoretti, Guillaume},
  booktitle     = aaai,
  year          = {2025}
}

@inproceedings{zhang2025d4orm,
  title         = {D4orm: Multi-Robot Trajectories with Dynamics-aware Diffusion Denoised Deformations},
  author        = {Zhang, Yuhao and Okumura, Keisuke and Woo, Heedo and Shankar, Ajay and Prorok, Amanda},
  year          = {2025},
  booktitle     = iros
}

@inproceedings{intelligence2026pi,
  title         = {$\pi^\ast_{0.6}$: a {VLA} That Learns From Experience},
  author        = {{Physical Intelligence}},
  booktitle     = rss,
  year          = {2026}
}

@inproceedings{jain2026pairwise,
  title         = {Pairwise is Not Enough: Hypergraph Neural Networks for Multi-Agent Pathfinding},
  author        = {Jain, Rishabh and Okumura, Keisuke and Amir, Michael and Lio, Pietro and Prorok, Amanda},
  year          = {2026},
  booktitle     = iclr
}

@inproceedings{jain2026graph,
  title         = {Graph Attention-Guided Search for Dense Multi-Agent Pathfinding},
  author        = {Jain, Rishabh and Okumura, Keisuke and Amir, Michael and Prorok, Amanda},
  year          = {2026},
  booktitle     = aaai
}

@inproceedings{moldagalieva2026dblacam,
  title         = {db-LaCAM: Fast and Scalable Multi-Robot Kinodynamic Motion Planning with Discontinuity-Bounded Search and Lightweight MAPF},
  author        = {Moldagalieva, Akmaral and Okumura, Keisuke and Prorok, Amanda and Hönig, Wolfgang},
  year          = {2026},
  booktitle     = icaps
}

@string{access={IEEE Access}}

@string{aij   = {Artif. Intell.}}

@string{ijrr  = {Int. J. Robot. Res.}}

@string{tro   = {IEEE Trans. Robot.}}

@string{ral   = {IEEE Robot. Autom. Lett.}}

@string{auro={Auton. Robots}}

@string{aaai={AAAI}}

@string{aamas={AAMAS}}

@string{icaps={ICAPS}}

@string{neurips={NeurIPS}}

@string{corl={CoRL}}

@string{iros={IROS}}

@string{icra={ICRA}}

@string{iclr={ICLR}}

@string{rss={RSS}}

@string{isrr={ISRR}}

@string{ecc={ECC}}

@string{isrr={Proceedings of the International Symposium of Robotics Research (ISRR)}}

\end{document}